\documentclass[10pt,twocolumn,letterpaper]{article}

\usepackage{cvpr}              

\usepackage{multicol}
\usepackage{multirow}
\usepackage{wrapfig}
\usepackage{amsthm}
\newtheorem{proposition}{Proposition}[section]

\definecolor{cvprblue}{rgb}{0.21,0.49,0.74}
\usepackage[pagebackref,breaklinks,colorlinks,allcolors=cvprblue]{hyperref}

\def\paperID{***} 
\def\confName{CVPR}
\def\confYear{2026}

\title{Intervention-Aware Multiscale Representation Learning from Imaging Phenomics and Perturbation Transcriptomics}

\author{Jiayuan Chen\qquad Ruoqi Liu \qquad Zishan Gu \qquad Ping Zhang\thanks{Corresponding author.}\\
The Ohio State University\\
{\tt\small {\{chen.12930,liu.7324,gu.1090,zhang.10631\}@osu.edu}}
}

\begin{document}
\maketitle
\begin{abstract}

Microscopy-based phenotypic profiling is scalable for drug discovery but lacks the mechanistic depth of transcriptomics, which remains costly and scarce. Existing multimodal approaches either use images to support other modalities or naively align representations by sample identity, ignoring cell-type and dose variations in weakly paired data-limiting generalization to unseen interventions. In this paper, we introduce an intervention-aware distillation framework that leverages perturbational transcriptomics to guide image representation learning. A transcriptome-conditioned teacher integrates gene expression and intervention metadata to produce soft distributions over a chemistry-aware codebook organized by drug similarity. The teacher employs a fine-tuned single-cell foundation model to encode cell-type context and disentangle dose effects. An image-only student learns to predict these distributions from microscopy alone, distilling mechanistic knowledge while operating independently at test time. This design emphasizes intervention semantics rather than identity alignment and explicitly handles dose and cell-type mismatches.
We provide theoretical guarantees showing that transcriptomic guidance tightens the risk bound for image-based prediction. On Cell Painting and RxRx datasets paired with L1000, our method significantly improves one-shot transfer to unseen interventions and drug-target gene discovery compared to self-supervised and alignment baselines\footnote{https://github.com/The-Real-JerryChen/BioMicroscopyProfiler}.

\end{abstract}    
\section{Introduction}
Phenotype-based screening \cite{celik2022biological,rood2024toward} has become central to drug discovery \cite{hao2024large,moshkov2024learning} and to building virtual cells \cite{bunne2024build,zhang2025tahoe} that predict cellular responses to perturbations. Two complementary readouts dominate this landscape: microscopy imaging \cite{bray2016cell} provides scalable, low-cost morphological profiles suitable for screening thousands of compounds, while perturbational transcriptomics \cite{subramanian2017next,zhang2025tahoe} captures pathway-level mechanistic signatures that reveal how drugs modulate cellular programs. In current practice, image-based models rely primarily on self-supervised learning (SSL) methods such as masked autoencoders (MAE) \cite{kraus2024masked,he2022masked} or DINO \cite{caron2021emerging,kenyon2024vitally} to extract transferable visual features. While these approaches learn general-purpose representations, the resulting features remain weakly connected to the biological mechanisms and transcriptomic changes that drugs actually induce \cite{bendidi2025cross}. They excel at capturing morphological patterns but lack the mechanistic grounding that transcriptomics provides-the very information needed to understand why a phenotype emerges and to predict responses to novel interventions. This raises a fundamental question: when limited paired image-transcriptome data becomes available, how can we leverage perturbational transcriptomics to guide image representation learning toward mechanistic understanding?

Existing image-centric multimodal approaches \cite{replogle2022mapping} typically leverage transcriptomic or chemical information to improve compound representations. Some methods \cite{wang2023removing,liu2024learning,Chen_2025_ICCV} align both images and RNA-seq to a shared drug embedding space, while others \cite{sanchez2023cloome,zheng2024cross,fradkin2024molecules} directly pair drug structure with image features, bypassing RNA altogether. While these designs enable cross-modal retrieval, they fundamentally center on drug identity rather than intervention effects. By treating drug identity as the primary supervision signal, they collapse the effect variations introduced by dose and cell type into binary matching objectives-the same compound assayed under different conditions is treated as an equivalent positive pair, discarding the graded response information critical for generalization. Moreover, these approaches primarily aim to learn better drug \cite{sanchez2023cloome,liu2024learning} or transcriptomic representations \cite{bendidi2025cross}, with images serving mainly as auxiliary signals rather than the target modality for improvement. As a result, the learned image features remain disconnected from the mechanistic pathways that link drug chemistry to transcriptomic response and ultimately to morphological phenotype. Models trained this way struggle to transfer to unseen perturbations and cellular contexts where intervention effects differ from the training distribution \cite{kraus2025rxrx3,chandrasekaran2024three}.

We propose to reframe multimodal learning around intervention semantics rather than sample identity. Drug perturbations follow a structured causal pathway: chemistry determines which pathways are engaged, shaping transcriptomic responses that ultimately manifest in morphological phenotypes. To leverage this structure, we introduce an intervention-aware distillation framework in which perturbational transcriptomics guides image representation learning during training. A transcriptome-conditioned teacher model encodes intervention effects from RNA-seq and metadata, while an image-only student learns to recover these effects from microscopy alone. Two technical designs instantiate this principle. First, we fine-tune a single-cell foundation model (scFM) on perturbation response prediction to encode cell-type context from basal gene expression and intervention-specific information from perturbed transcriptomes and doses. This encoding provides mechanism-aware representations that capture pathway-level effects while accounting for dosage variations across weakly paired samples. Second, to organize supervision by mechanism rather than identity, we construct a chemistry-aware codebook by projecting drug molecular fingerprints into a normalized prototype space. The teacher integrates both RNA-seq and paired images to produce soft distributions over this codebook, while the student learns to predict these distributions from images alone. Transcriptomics thus serves as privileged information during training, shaping image features toward mechanistic understanding, but is not required at test time.
We also provide theoretical analysis showing that incorporating transcriptomic signals during training provably tightens the risk bound for image-based intervention prediction.

We validate our approach on multiple benchmarks spanning two encoder backbones. On Cell Painting~\cite{bray2017dataset,chandrasekaran2024three} and RxRx~\cite{kraus2025rxrx3} datasets paired with L1000 perturbational transcriptomics, we evaluate one-shot transfer to unseen interventions and unsupervised drug-target gene discovery. Compared with well-established self-supervised baselines (MAE, DINO) and contrastive alignment methods, our intervention-aware framework achieves significant improvements across all tasks. In summary, our contributions are:
\begin{itemize}[leftmargin=*]
\item We introduce an intervention-aware distillation framework that uses perturbational transcriptomics to guide image representation learning toward mechanistic understanding, with theoretical guarantees showing that transcriptomic guidance tightens the risk bound for image-based intervention prediction.
\item We propose a chemistry-aware codebook that structures supervision by drug mechanism, combined with a transcriptome-conditioned teacher that leverages a fine-tuned scFM to encode cell context and intervention effects.
\item We demonstrate strong empirical performance on one-shot transfer to unseen interventions and unsupervised drug-target discovery across Cell Painting and RxRx datasets, outperforming self-supervised and alignment baselines.
\end{itemize}

\section{Related Work}

\subsection{Phenotypic Profiling of Drug Perturbations}
Microscopy-based profiling, such as Cell Painting~\cite{bray2016cell}, uses multiplexed fluorescent imaging to capture morphological responses to drug perturbations at scale, enabling applications in novel biological relationship discovery \cite{kraus2025rxrx3,chandrasekaran2024three,sivanandan2023pooled}. These assays are cost-effective and high-throughput but lack direct mechanistic interpretation. Recent work applies SSL, such as  MAE~\cite{kraus2024masked} and DINO~\cite{kenyon2024vitally}, to extract morphological features from unlabeled images, though these remain weakly connected to biological mechanisms. Perturbational transcriptomics captures genome-wide gene expression changes, providing pathway-level mechanistic insights through datasets like LINCS L1000~\cite{subramanian2017next} and Tahoe-100M~\cite{zhang2025tahoe}. scFMs such as scGPT~\cite{hao2024large} and Geneformer~\cite{theodoris2023transfer} can be fine-tuned to predict perturbation responses~\cite{peidli2024scperturb}, offering rich mechanistic representations but at higher cost and lower throughput. These modalities are complementary, images provide scale, transcriptomics provides mechanism, yet paired datasets remain scarce and often weakly paired \cite{bendidi2025cross}, with the same drug profiled at different doses, batch, or cell types across modalities.

\subsection{Multimodal Learning for Perturbations}
Recent work explores molecular representation learning from cellular data for cross-modal retrieval and molecular property prediction. Most methods employ alignment objectives: aligning drug structures with cellular images~\cite{sanchez2023cloome,wang2023removing,zheng2024cross}, or integrating gene expression with other omics to provide multi-faceted drug profiles~\cite{liu2024learning,wang2023removing}. However, these approaches primarily aim to learn better drug or transcriptomic representations for downstream molecular tasks, such as predicting chemical properties and drug-drug interactions, with images serving mainly as auxiliary signals to enrich compound embeddings rather than as the target modality for improvement. A recent complementary direction distills image foundation models to improve transcriptomics representations~\cite{bendidi2025cross}, but the reverse that using transcriptomics to guide image learning remains underexplored. Critically, existing alignment methods center on drug identity rather than intervention effects, collapsing dose and cell-type variations into binary matches and struggling with weakly paired data. We instead propose intervention-aware distillation that using transcriptomics  to guide image representations toward mechanistic understanding.
\section{Method}

\subsection{Problem Statement}
In this paper, we study cellular responses to drug interventions observed through two complementary phenotypic readouts: cellular imaging $I$ and transcriptomics (RNA-seq) $R$.
\begin{wrapfigure}{r}{0.46\linewidth}  
  \centering
  \includegraphics[width=\linewidth]{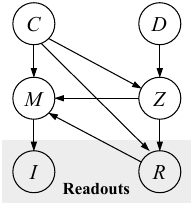}
  \vspace{-8pt}
  \caption{Causal pathway of drug interventions across modalities.}
  
  \label{fig:causalgraph}
\end{wrapfigure}

\noindent\textbf{Causal Structure}. \autoref{fig:causalgraph} depicts the multiscale cellular state under a drug intervention. Let $D$ denote the molecular structure of the interventional drug; after exposure, the drug binds to its molecular target(s), producing a target-engagement/signaling state $Z$. This perturbed state then drives changes at two levels: the transcriptome $R$ and the morphology $M$. Path $Z \!\to\! R$ corresponds to target binding perturbing signaling cascades and transcription-factor activity, inducing gene-expression changes captured by mRNA-seq. Morphology $M$ is shaped both by direct drug-induced protein/biophysical effects via path $Z \!\to\! M$ and by indirect, transcription-mediated remodeling in which expression programs adjust protein levels/modifications and reorganize cellular architecture via the path $Z \!\to\! R \!\to\! M$. Finally, morphology is observed as images $I$ through the optical pipeline that contains staining, imaging, and preprocessing. Cell type $C$ provides the biological context, shaping target engagement $Z$, basal transcriptomic state $R$, and intrinsic morphology $M$, thereby introducing three conditioning paths, $C\!\to\!Z$, $C\!\to\!R$, and $C\!\to\!M$, that modulate the drug-induced cascades.

\noindent\textbf{Scope and Assumptions.}  We focus on static and do not explicitly model temporal dynamics or multi-cellular interactions in our primary framework. While batch effects are a practical concern in multi-plate imaging and cross-laboratory RNA-seq data, we do not model them explicitly in our causal graph; we discuss mitigation strategies and their relationship to our framework in Appendix.

\noindent\textbf{Task Definition}. Given this setup, we target image-based drug discovery: learning image representations that enable drug identification, mechanism clustering, and generalization to unseen compounds and cell types. Formally, we aim to estimate  $P(D\,|\,I,C)$, the posterior distribution over drugs given cellular images and cell type, which captures how morphology encodes intervention-specific signals. At test time, only images $I$ and cell type $C$ are available; transcriptomics $R$ serves as privileged information during training to guide the image encoder toward mechanistic understanding.

\begin{figure*}[t]
    \centering
    \includegraphics[width=0.85\linewidth]{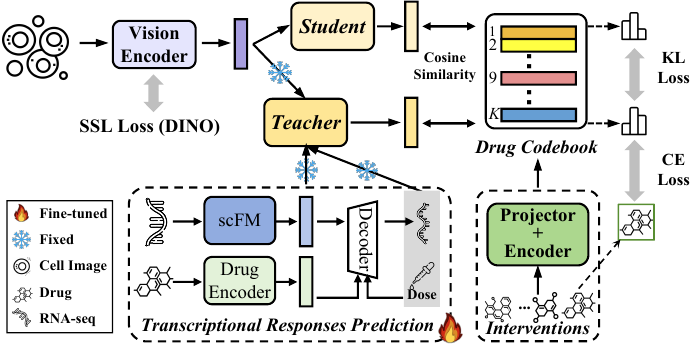}
        \vspace{-4pt}
    \caption{Framework overview. Our approach uses a chemistry-aware codebook and transcriptome-conditioned teacher to guide image student learning through distillation. Only images are required at test time. }

    \label{fig:framework}
    \vspace{-10pt}
\end{figure*}

\subsection{Transcriptome-Guided Image Learning}

We now establish the theoretical foundation for using transcriptomic guidance to improve image-based intervention prediction. Our goal is to learn a predictor $S_T(D | I, C)$ that estimates the drug (intervention) posterior from images and cell type alone. During training, however, we have access to paired transcriptomes $R$ that capture perturbation-induced gene expression changes: responses that reflect pathway-level mechanisms not directly observable in morphology. The key question is whether and how incorporating transcriptomic information during training provably improve the learned image representations

We formalize this through a risk-bound analysis on the training distribution. Define $\mathcal{L}[q]$ as the expected log-loss of a predictor $q$ that depends only on $(I, C)$, where $\mathcal{P}$ is the joint distribution of observed training data. Consider a richer predictor $T(D | I, R, C)$ that incorporates RNA during training, producing a probability distribution over drugs $D$. The following proposition shows that training with transcriptomics provably tightens the risk bound.
\begin{proposition}[Transcriptome-Guided Learning Bound]
\label{prop:transcriptome-bound}
Given readout variables $I$ and $R$ conditioned on cell type $C$, consider a predictor $T: (I,R,C) \mapsto \Delta(\mathcal{D})$ that outputs a distribution over drugs $D$, where $\Delta(\mathcal{D})$ denotes the probability simplex. Define the marginalized predictor that depends only on $(I,C)$:
$$S_T(D|I,C) = \mathbb{E}_{R|I,C}[T(D|I,R,C)]$$
Let $\mathcal{L}[q] = \mathbb{E}_{(I,C,D) \sim P}[-\log q(D|I,C)]$ denote the log-loss risk for any predictor $q$ that depends only on $(I,C)$. Then $\mathcal{L}[S_T]$ is upper bounded by:
$$H(D|I,C) + \mathbb{E}_{(I,R,C)}[\mathrm{KL}(P(D|I,R,C) \| T(\cdot|I,R,C))],$$
where $H(D|I,C)$ is the conditional entropy and the second term is the training objective that can be directly optimized.
\end{proposition}
This bound establishes that fitting the conditional distribution $P(D | I, R, C)$ during training, and marginalizing over $R$ at test time, yields a tighter risk bound on the training distribution than directly fitting $P(D | I, C)$ without RNA guidance. The key insight is that transcriptomic responses provide mechanistic information: conditioning on gene expression allows the teacher $T$ to distinguish drugs that produce similar morphologies but engage different molecular pathways. When this knowledge is distilled into the image student through marginalization, the student learns representations that align with biological mechanisms rather than superficial visual patterns. The complete proof is provided in the Appendix.

The proposition suggests three design principles for our framework: (i) model $P(D | I, R, C)$ in a way that respects the causal structure of interventions, (ii) explicitly condition on cell type and dose to capture context-specific perturbation effects, and (iii) train the image student to approximate the marginalized posterior $S_T$. Note that the bound concerns the training distribution-it establishes that transcriptomic guidance improves learning on observed data but does not directly guarantee generalization to unseen drugs. However, we hypothesize that the mechanistic features learned through this process will transfer to novel compounds with similar mechanisms, which we validate empirically in Section 4. We instantiate these principles in the following section.

\subsection{Transcriptome-Guided Distillation Framework}

We now present our transcriptome-guided distillation framework, which operationalizes the theoretical insights from Section 3.2. \autoref{fig:framework} illustrates the overall architecture. Our framework consists of three key components. First, we construct a chemistry-aware codebook that organizes the drug space into $K$ prototypes based on molecular representations, providing a structured target space where similar drugs are clustered. Second, we design a transcriptome-conditioned teacher that integrates paired images $I$, RNA-seq $R$, and intervention metadata (cell type $C$, doses $\delta_I, \delta_R$) to produce soft distributions over the codebook, capturing mechanistic signals from transcriptomics. Third, an image-only student learns to predict the same distributions from images alone, distilling the teacher's multimodal knowledge. We now detail each component.

\subsubsection{Chemistry-aware Codebook}
We construct a codebook that organizes drugs by molecular similarity. Given the set of $K$ training drugs, we extract a molecular representation for each drug. These representations form the rows of a drug feature matrix, which we project into the same embedding space as the teacher and student encoders using a lightweight MLP followed by $\ell_2$ normalization. This yields a prototype matrix $V \in \mathbb{R}^{K \times d}$, where each row $\mathbf{v}_k$ represents a drug prototype on the unit hypersphere that are fully learnable parameters updated through backpropagation during training. The chemistry-aware structure allows models to predict soft distributions over prototypes rather than hard drug identities, facilitating transfer to novel compounds through chemical similarity. While the prototypes adapt to encode mechanistic similarity as the teacher learns to predict codebook distributions from both images and transcriptomics. 

\subsubsection{Transcriptome-conditioned Teacher}
The teacher model $T_{\theta}$ takes paired multi-modal observations and produces distributions over the codebook that reflect intervention-specific mechanistic signatures. A key design principle is that the teacher does not directly encode the drug's molecular structure; instead, it must infer pathway-level perturbation effects from the observed cellular responses (morphology and transcriptomics), and supervision from the codebook provides the link between inferred mechanisms and drug chemistry. This avoids shortcut learning where the teacher trivially matches drug identity.

To obtain robust encodings of cell type and dose-dependent transcriptomic effects, we employ a scFM~\cite{hao2024large}, fine-tuned on perturbation response prediction. During fine-tuning, the model takes as input a cell's basal gene expression $R_{\text{basal}}$, a drug representation, and a dose value $\delta_r$, and predicts the post-perturbation expression profile. Specifically, the scFM encodes $R_{\text{basal}}$ and combines it with drug and dose embeddings to form a latent state, which is then decoded to reconstruct the perturbed transcriptome. After fine-tuning, we use the trained scFM encoder to extract two types of representations: (i) a cell-type encoding $\mathbf{h}_C$ from basal gene expression, providing stable biological context across perturbations; and (ii) a dose-conditioned transcriptome encoding that disentangles dose information from the observed RNA-seq measurements. The scFM's training objective ensures these encodings capture interpretable perturbation-response structure. Implementation details are provided in Appendix.

The teacher then encodes three sources of information:
\begin{itemize}[leftmargin=*]
\item \textbf{Image with dose}: We encode the cellular image $I$ and then detach dose-specific information $\delta_I$ from the image features using a FiLM-style conditioning mechanism~\cite{perez2018film}, producing a dose-aware image representation $\mathbf{h}_I$.
\item \textbf{RNA with dose}: We encode the observed perturbed transcriptome $R$ along with its associated dose $\delta_R$, yielding $\mathbf{h}_R$. The dose disentanglement allows this encoding to separately represent the perturbation magnitude.
\item \textbf{Cell type}: We use the basal gene expression encoding $\mathbf{h}_C$ from the scFM to represent cell-type-specific context.
\end{itemize}
Critically, by encoding image dose $\delta_I$ and RNA dose $\delta_R$ separately within their respective modality features, the teacher can handle weakly paired data where doses differ across modalities. Rather than treating mismatched doses as identical through binary label matching, the teacher learns to reason about dose-dependent effects explicitly, producing appropriate soft targets that reflect the actual dosage contexts.

For the codebook prediction,  the teacher fuses the three encoded features through concat and a projection network $f_t$:
$$\mathbf{h}_{t} = f_{t}([\mathbf{h}_I|| \mathbf{h}_R||\mathbf{h}_C]) \in \mathbb{R}^{d}.$$
It then computes similarity scores with all codebook prototypes, and produces a soft distribution over the codebook via temperature-scaled softmax:

$$P_t= \text{Softmax}\left( \frac{V\cdot\mathbf{h}_{t}}{\tau\cdot|\mathbf{h}_{t}|_2}\right) \in \mathbb{R}^K$$
where  $V \in \mathbb{R}^{K \times d}$ is drug codebook and $\tau$ controls distribution sharpness.  The teacher is trained with cross-entropy loss $\mathcal{L}_{teacher}$ against true drug labels, establishing the connection between observed cellular responses and drug chemistry through the codebook. The learned soft distributions, which encode fine-grained mechanistic similarities beyond hard labels, serve as targets for student distillation.

\subsubsection{Image Student and Distillation Objective}
The student learns to predict codebook distributions from images alone, without explicit cell type input. This design reflects the deployment scenario where only microscopy is available, and is justified by the fact that cell type information is inherently encoded in morphological features. The student uses the shared image encoder from the teacher, followed by a projection head to get the feature $\mathbf{h}_s$, producing codebook distributions $P_s$ from images alone. We train the student to match the teacher's distributions via KL divergence:
$$\mathcal{L}_{\text{distill}} = \mathbb{E}_{(I,R,C,\delta)} \left[ \text{KL}(P_t | P_s \right].$$

Our framework is compatible with any self-supervised method. We apply $\mathcal{L}{\text{ssl}}$ (e.g., DINO, MAE) to all images, paired and unpaired, leveraging abundant imaging data while learning general morphological features. The complete objective is: $$\mathcal{L} = \mathcal{L}_{\text{distill}} +\alpha \cdot   \mathcal{L}_{\text{teacher}} +\beta \cdot \mathcal{L}_{\text{ssl}}.$$

Our framework design operationalizes Proposition 3.1 by enabling the student to learn mechanistic features through two complementary mechanisms: the chemistry-aware codebook provides transfer to novel compounds via chemical similarity, while the transcriptome-conditioned teacher captures pathway-level perturbation effects that generalize beyond training drugs. By explicitly handling dose variations and avoiding identity shortcuts, the framework learns transferable representations from limited paired data. We empirically validate this generalization in Section 4.

\begin{table}[t]
    \centering
    \begin{tabular}{l|rrr}
    \toprule
    Dataset     & RxRx3  & CPG-Pilot &CPG-12  \\
    \midrule
    Image     &  61,690 & 93,696 &916,721 \\
    Paired RNA & 3,682 & 1,883 &22,066 \\
    \# Drugs & 1,662 & 302 &30,340\\
    \# Overlaps & 235 & 85 &6,989\\
    \# Eval Intervention & 736 & 260 & 121  \\
    \bottomrule
    \end{tabular}
    \vspace{-4pt}
    \caption{Dataset statistics and imaging-transcriptomics pairing. Numbers indicate image samples, RNA profiles, and compound overlap across modalities.}
    \vspace{-12pt}
    \label{tab:placeholder}
\end{table}

\begin{table*}[t]
    \centering
    \begin{tabular}{l|rr|rr|rr|rr|rr}
    \toprule
     Dataset & \multicolumn{4}{c|}{RxRx3} & \multicolumn{4}{c|}{CPG-Pilot} & \multicolumn{2}{c}{CPG-12} \\
     \midrule
     Backbone & \multicolumn{2}{c|}{ViT} & \multicolumn{2}{c|}{CA-ViT} & \multicolumn{2}{c|}{ViT} & \multicolumn{2}{c|}{CA-ViT} & \multicolumn{2}{c}{ViT} \\
     Metric   & Top-1 & Top-5 & Top-1 & Top-5& Top-1 & Top-5& Top-1 & Top-5& Top-1 & Top-5\\
     \midrule
     MAE & 1.54  &6.79 & 1.00 & 3.26 & 2.95 & 12.31 & 2.05 & 8.97 & 4.68 & 11.02\\
     DINO & 4.82 & 14.86 &4.62 &13.32 &11.36 & 30.82 & 13.97 & 33.46 & 36.36&60.01\\
     CL (D)& \underline{5.16} & \underline{15.63}  & \underline{5.93}&  \underline{16.40}& \underline{14.61}& \underline{36.66} & \underline{15.17} & \underline{38.46} & \underline{40.04} & \underline{65.07}\\
     CL (R)& 4.93 & 14.27 & 5.25 & 13.90&13.97&30.64& 14.87 & 35.51 &30.57 & 60.89\\
     \midrule
     TIDE & \textbf{5.62} & \textbf{16.86}& \textbf{6.04} & \textbf{17.27} & \textbf{16.01} & \textbf{39.12}&  \textbf{17.80} & \textbf{40.74} &\textbf{42.09} & \textbf{71.07}\\
    \bottomrule
    \end{tabular}
    \vspace{-6pt}
    \caption{One-shot transfer to unseen interventions. We report Top-1 and Top-5 accuracy (\%) for transfer to unseen genetic perturbations (CPG-Pilot, RxRx3) and unseen compounds (CPG-12).  Best results highlighted are in bold, and second-best results are underlined}
    \label{tab:one-shot}
    \vspace{-12pt}
\end{table*}

\section{Experiments}
We evaluate our framework on image-based drug discovery tasks to assess whether transcriptomic guidance improves (i) generalization to novel interventions and (ii) recovery of mechanistic drug-target relationships. We conduct experiments on three large-scale Cell Painting datasets paired with perturbational transcriptomics (L1000), spanning multiple cell lines and thousands of compounds. We compare against self-supervised baselines (MAE, DINO) and alignment-based methods (CLIP-style contrastive learning) to demonstrate the benefit of our approach. 
\subsection{Datasets}

\textbf{Imaging readouts.} We use three Cell Painting datasets representing different scales and perturbation types: Cell Painting and RxRx.
\begin{enumerate}
    \item \textbf{JUMP-Pilot (CPG-Pilot)~\cite{chandrasekaran2024three}} profiles A549 and U2OS cells under both chemical and genetic perturbations with standardized Cell Painting protocols. The joint coverage of compounds and CRISPR perturbations enables evaluation of mechanism-level relationships and chemical-genetic concordance.
    \item \textbf{CPG0012 (CPG-12)~\cite{bray2017dataset}} is a large-scale collection of $\sim$30,000 bioactive compounds in U2OS cells with consistent re-profiling. Its scale and diversity support evaluation of cross-compound generalization and large-scale drug discovery applications.
    \item \textbf{RxRx3-core (RxRx3)~\cite{kraus2025rxrx3}} is a research-friendly subset of RxRx3 with six-channel imaging in HUVEC cells, covering both drug perturbations and CRISPR knockouts. The shared indexing across chemical and genetic conditions facilitates drug-gene association analyses.
\end{enumerate}

\noindent\textbf{Transcriptomic readouts}. 
We pair imaging data with L1000~\cite{subramanian2017next}, a high-throughput perturbational transcriptomics platform that measures 978 landmark genes. L1000 provides bulk-cell gene expression profiles, matching the population-level nature of our imaging readouts. We extract L1000 profiles for compounds in the same cell lines as our imaging datasets, enabling transcriptome-guided image learning. To leverage L1000 data with our framework, we fine-tune a single-cell foundation model (scGPT~\cite{hao2024large}) on perturbation response prediction using L1000 measurements. Although scGPT is pretrained on single-cell atlases, we adapt it to bulk-cell L1000 data and training the model to predict perturbed expression given basal profiles, drug representations, and doses. This fine-tuned scFM provides the cell-type and dose encoders used in our teacher model (Section 3.3.2).

Table 1 summarizes the pairing between imaging and transcriptomic data. Importantly, the pairing is weakly paired: images and RNA-seq measurements share the same cell type and compound but often differ in dose, reflecting realistic data collection constraints. This weak pairing is a key challenge our framework addresses through explicit dose conditioning.

\subsection{Tasks and Evaluation Metrics}
We evaluate on two complementary tasks that assess generalization and mechanistic understanding:
\subsubsection{Task 1: One-shot transfer to unseen intervention}
This task evaluates whether learned image representations generalize to novel perturbations in a low-data regime. We follow a metric-based few-shot learning protocol: given a pretrained and frozen image encoder, we initialize a cosine similarity-based classifier from 1-shot support examples (one image per unseen intervention class), then optimize only the classification head on query images while keeping the encoder frozen. Features and class prototypes are L2-normalized, and predictions use temperature-scaled cosine similarity.

The specific unseen interventions differ by dataset characteristics. On CPG-Pilot and RxRx3, we evaluate transfer to unseen genetic perturbations (CRISPR knockouts), as these datasets contain limited chemical diversity but extensive genetic screens measured with the same imaging protocol. On CPG-12, which contains more than 30,000 diverse compounds, we evaluate transfer to unseen compounds by holding out a subset of drugs. Both settings assess the core capability of generalizing features to novel interventions.

\subsubsection{Task 2: Unsupervised drug-target gene discovery}
This task evaluates whether image representations capture mechanistic relationships by testing drug-target gene association without explicit supervision. Following ~\citet{kraus2025rxrx3}, we compute embeddings for compound-treated images and gene knockout images, score all drug-gene pairs by cosine similarity, and evaluate against curated ground-truth associations derived from each dataset's benchmark annotations, which is further filtered for high-confidence links. We report results averaged over 100 random seeds.

On RxRx3, compounds appear at multiple doses and may target multiple genes. We compute per-dose drug embeddings (averaged over replicate images), score associations at each dose, and report dose-averaged Average Precision (AP) and AUC. On CPG-Pilot, compounds are profiled at a single dose with a single primary target gene; we report AP and Hit@k. This task directly measures whether the learned representations reflect biological mechanisms, a key advantage of transcriptomic guidance.

\subsection{Baselines and Implementation}
We refer to our framework as TIDE (Transcriptome-Informed Distillation for image Encoding) and compare against two families of methods:

\noindent \textbf{SSL Methods}: MAE~\cite{he2022masked} and DINO~\cite{caron2021emerging}, widely adopted for cellular imaging. These methods learn from images alone without transcriptomic or drug information.

\noindent \textbf{Alignment-based contrastive learning}: We implement CLIP-style approaches that align image embeddings to auxiliary modalities:
CL(D): Following recent phenomolecular retrieval work~\cite{sanchez2023cloome,fradkin2024molecules}, aligns images to drug molecular representations;
CL(R): Aligns images and transcriptomic readouts based on shared drug perturbations \cite{wang2023removing,liu2024learning}. To ensure fair comparison, we evaluate hybrid variants that combine contrastive alignment with DINO self-supervision, which typically improve performance. 

\noindent \textbf{Vision Encoder}. All methods use the same backbone architectures for fair comparison. We evaluate two vision transformers: a standard ViT~\cite{dosovitskiy2020image} and a Channel-Agnostic ViT (CA-ViT)~\cite{kraus2024masked}, a cellular imaging-specific architecture that achieves comparable performance to ViT. Model sizes are matched to dataset scale: ViT-Small for CPG-Pilot and RxRx3, ViT-Base for CPG-12 (CA-ViT-Base is omitted on CPG-12 due to computational cost).

All methods train for 400 epochs with matched augmentations and optimizers within each method family. We select checkpoints via k-nearest neighbor (k=20) evaluation on a held-out validation split (10 images per perturbation class). For Task 1, we split interventions (not samples) into train/val/test to ensure true generalization evaluation. For Task 2, we use all available data to compute embeddings and evaluate on curated drug-target associations. Detailed hyper-parameters and split information are provided in Appendix.

\begin{table}[t]
    \centering
    \resizebox{0.49\textwidth}{!}{
    \begin{tabular}{l|cccc|cccc}
    \toprule
     Backbone    & \multicolumn{4}{c|}{ViT} & \multicolumn{4}{c}{CA-ViT}\\
     \midrule
     Dataset & \multicolumn{2}{c}{RxRx3} & \multicolumn{2}{c|}{CPG-Pilot}&  \multicolumn{2}{c}{RxRx3} & \multicolumn{2}{c}{CPG-Pilot}\\
     Metric & AP & AUC & AP & Hit@5 & AP & AUC & AP & Hit@5\\    \midrule
     MAE &  0.256  & 0.538 & 0.079&0.103& 0.245 &0.518 &0.097 &0.124\\
     DINO & \underline{0.300} & \underline{0.633} & \underline{0.109} & \underline{0.135} & 0.252 & \underline{0.545} & \underline{0.111} & \underline{0.154}\\
     CL (D)& 0.252 & 0.537 & 0.099 &0.127 & \underline{0.255} & 0.540 &0.103 &0.131\\
     CL (R)& 0.247 & 0.532 & 0.105 & 0.134 & 0.250 & 0.540 &0.105 & 0.135\\
     \midrule
     TIDE  & \textbf{0.317} & \textbf{0.640} & \textbf{0.119} & \textbf{0.149} & \textbf{0.301} & \textbf{0.612} & \textbf{0.120} & \textbf{0.171}\\
     \bottomrule
    \end{tabular}}
    \vspace{-4pt}
    \caption{Unsupervised drug-target gene discovery. We report AP and retrieval metrics without explicit supervision on drug-target associations. Best results in bold.}
    \vspace{-12pt}
    \label{tab:1}
\end{table}
\subsection{Main Results}
We design experiments to answer two research questions that validate TIDE's core contributions:
\begin{itemize}[leftmargin=*]
\item \textbf{RQ1}: Does transcriptomic guidance improve generalization to unseen interventions?
\item \textbf{RQ2:} Do learned image representations capture biologically meaningful mechanisms?
\end{itemize}

\autoref{tab:one-shot} presents Top-1 and Top-5 accuracy on one-shot classification. All results are averaged over 50 runs. TIDE achieves the best performance across all settings, substantially outperforming all baselines. CL(D) attains the second-best results among baselines, which together with TIDE's performance validates the importance of incorporating perturbation information for model training. CL(R) performs slightly better than SSL baselines but falls short of drug-based alignment, as the limited availability of paired transcriptomic data constrains its training signal compared to CL(D)'s ability to align with molecular representations for all drugs. Among SSL methods, DINO substantially outperforms MAE across all datasets. Although training for the same number of epochs, DINO's use of local and global crop augmentations results in higher computational cost per epoch, effectively providing richer training signal. Comparing backbones, ViT and CA-ViT show comparable overall performance on RxRx3 and CPG-Pilot across most methods. However, with TIDE CA-ViT demonstrates superior results, as its channel-agnostic patchification design achieves better representational capacity under sufficient training. TIDE's consistent improvements across datasets and backbones demonstrate that transcriptomic guidance enables image models to learn mechanistically grounded features that generalize to novel interventions.

\autoref{tab:1} presents performance on unsupervised drug-target gene discovery. Our model significantly outperforms all baselines on both datasets, although overall performance remains modest across all methods. As reported in \cite{chandrasekaran2024three}, cross-modal perturbation pairing signals are inherently weak, only slightly above random. Most baseline results on RxRx3 are only marginally better than random, underscoring the task's difficulty and demonstrating that our model learns mechanistic features beyond visual supervision alone.
Unlike the one-shot transfer task, CL(D) shows no clear advantage over DINO or CL(R) in this setting. This reflects the nature of drug-target discovery that while molecular structure alignment helps CL(D) learn drug identity and chemical similarity, identifying biological targets requires pathway-level mechanistic understanding that chemical fingerprints do not directly encode. In contrast, even limited transcriptomic signals or morphological patterns learned through augmentation-rich SSL may capture aspects of cellular response more relevant to target identification, explaining their comparable performance to CL(D) despite the differences observed in Task 1. The substantial improvement TIDE achieves on this challenging task, where baselines barely exceed random performance, confirms that our framework successfully distills pathway-level mechanistic knowledge from transcriptomics into image representations, enabling identification of biologically meaningful drug-target relationships from morphology alone.

\subsection{Ablation Studies and Analysis}
We conduct ablation studies to dissect the contributions of TIDE's design components and validate key hypotheses. 
\subsubsection{Effect of Transcriptomic Data}
To validate whether richer transcriptomic guidance improves performance, we separately evaluate TIDE on A549 and U2OS cell lines in CPG-Pilot. These two cell lines have different overlap with L1000 gene expression profiles: A549 has more paired compounds than U2OS. By comparing performance across these two settings, we can assess whether increased transcriptomic coverage leads to better mechanistic learning. \autoref{tab:ab1} presents one-shot transfer accuracy and drug-target discovery performance for both cell lines on ViT. Examining the DINO baseline performance, we observe that both cell lines achieve comparable results on one-shot transfer, with U2OS even slightly outperforming A549. However, after training with TIDE, A549 that has more paired transcriptomic data, achieves better performance and larger improvements over its DINO baseline. On the drug-target discovery task, while A549 demonstrates advantages on this task even with DINO, it shows more improvements with TIDE compared to U2OS. These findings validate the importance of transcriptomic information: when sufficient paired data is available, TIDE can more effectively distill mechanistic knowledge from transcriptomics, leading to stronger gains in both generalization and mechanistic understanding.



\begin{figure}[h]
    \centering
    \begin{subfigure}[b]{0.49\linewidth}
        \centering
        \includegraphics[width=\linewidth]{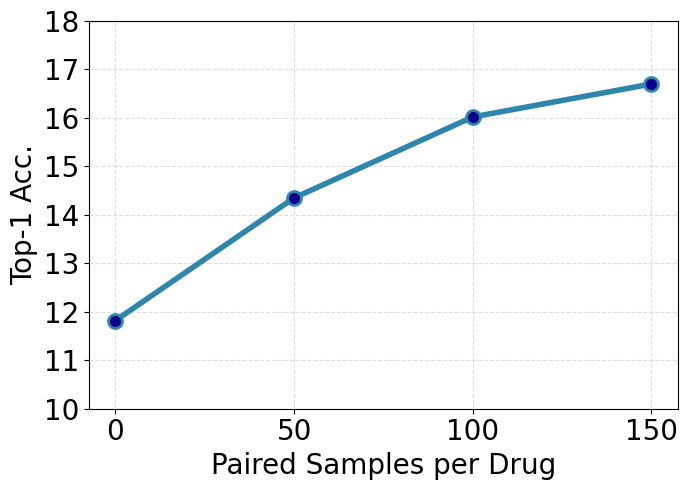}
        \caption{TIDE with fine-tuned scFM}
        \label{fig:ab2-a}
    \end{subfigure}
    \begin{subfigure}[b]{0.49\linewidth}
        \centering
        \includegraphics[width=\linewidth]{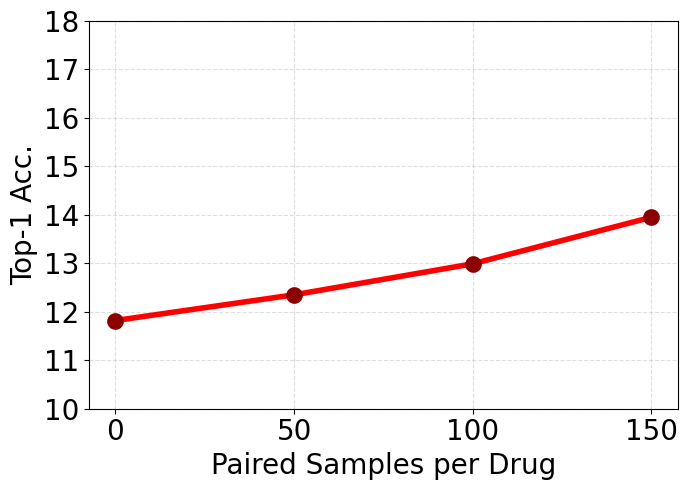}
        \caption{TIDE with pretrained scFM}
        \label{fig:ab2-b}
    \end{subfigure}
    \vspace{-6pt}
    \caption{Effect of paired sample size and scFM fine-tuning on one-shot transfer performance of TIDE. }
    \label{fig:ab2}
\end{figure}

\begin{table}[t]
    \centering
    \resizebox{0.495\textwidth}{!}{
    \begin{tabular}{l|cccc|cccc}
    \toprule
     Task    & \multicolumn{4}{c|}{one-shot transfer} & \multicolumn{4}{c}{drug-target discovery}\\
     \midrule
     Cell Line & \multicolumn{2}{c}{U2OS} & \multicolumn{2}{c|}{A549}&  \multicolumn{2}{c}{U2OS} & \multicolumn{2}{c}{A549}\\
     Metric & Top-1 & Top-5 & Top-1 & Top-5 & AP & Hit@5 & AP & Hit@5\\    \midrule
     DINO & 12.8 & 28.9 & 12.1 &28.2& 0.084 & 0.111 & 0.133  &0.159\\
     \midrule
     TIDE & 15.9 &36.9 &16.3 & 39.8& 0.092 & 0.120 &0.146 & 0.178\\

     \bottomrule
    \end{tabular}}
    \vspace{-8pt}
    \caption{Performance on CPG-Pilot separately for U2OS and A549.}
    \vspace{-16pt}
    \label{tab:ab1}
\end{table}

\subsubsection{Sample Efficiency and Effect of scFM Fine-tuning}
To understand how the quantity of paired data per drug affects performance and validate the necessity of scFM fine-tuning on perturbation response prediction, we vary the number of paired image-transcriptome samples per drug from 0 to 150 on CPG-Pilot. At 0 paired samples, TIDE reduce to DINO. We compare TIDE using a fine-tuned scFM encoder versus one with pretrained-only weights. \autoref{fig:ab2} presents the results. With fine-tuned scFM, TIDE shows substantial performance improvement as paired samples increase. This demonstrates that the fine-tuned encoder effectively extracts dose-specific and cell-context information from transcriptomic measurements, and more paired samples provide richer supervision for distillation. In contrast, TIDE using pretrained-only scFM shows only marginal improvement, despite also benefiting from the additional training images that accompany more paired samples. The minimal gain isolates the effect of increased image data alone and confirms that without fine-tuning on perturbation response prediction, the scFM encoder cannot effectively disentangle dose effects or encode cell-type-specific basal states.
\subsubsection{Analysis and Visualization}
To understand what TIDE's learned codebook encodes, we visualize drug representations in chemical fingerprint space versus learned codebook space. \autoref{fig:vis} presents a case study of two compounds, BMS-536924 and WH-4-023, profiled on A549 cells in CPG-Pilot. While the morphological images (left) show complex cellular responses with no obvious shared visual patterns to the human, the two drugs share a common mechanism of action: both are ATP-competitive tyrosine kinase inhibitors that modulate downstream signaling pathways. The middle panel shows UMAP projection of all CPG-Pilot drugs based on FCFP fingerprints \cite{rogers2010extended}. BMS-536924 and WH-4-023 are distant in chemical space, reflecting their different molecular structures, while in the learned codebook embedding space the two drugs are positioned much closer, indicating that the codebook has learned to group them by their shared mechanistic properties despite structural differences. This demonstrates that through end-to-end training with transcriptomic guidance, the codebook moves beyond chemical similarity to capture biological mechanism similarity, organizing drugs by how they perturb cellular pathways rather than merely by molecular structure. 

\begin{figure}[t]
    \centering
    \includegraphics[width=1\linewidth]{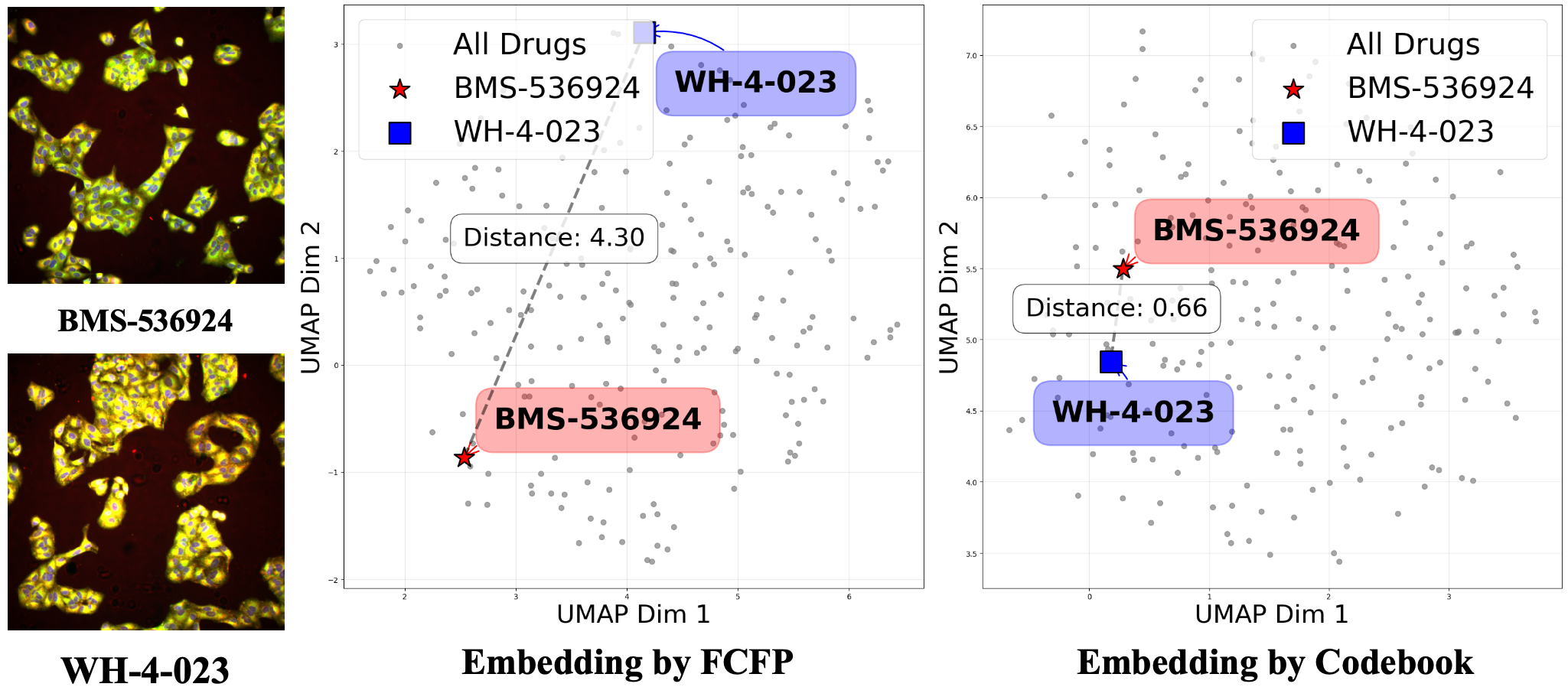}
    \vspace{-16pt}
    \caption{Left: Microscopy images of A549 cells treated with BMS-536924 and WH-4-023. Right: Embedding visualization by FCFP fingerprint v.s. learned codebook embedding}
    \label{fig:vis}
    \vspace{-12pt}
\end{figure}


\section{Conclusion}
We presented a transcriptome-guided distillation framework that uses perturbational RNA-seq as privileged information to guide image representation learning for drug discovery. By structuring supervision through a chemistry-aware codebook and explicitly conditioning on dose and cell type, our approach handles weakly paired data and avoids identity-based shortcuts that limit existing alignment methods. We provided theoretical analysis showing that transcriptomic guidance tightens the risk bound for image-based prediction, and demonstrated improvements over SSL and contrastive baselines on one-shot transfer to unseen interventions and unsupervised drug-target gene discovery.

\section*{Acknowledgments}
This work was funded in part by the National Science Foundation under award number IIS-2145625 and by the National Institutes of Health under awards number R01AI188576 and R01CA301579.

{
    \small
    \bibliographystyle{ieeenat_fullname}
    \bibliography{main}
}

\clearpage

\appendix

\section{Proof of Proposition 3.1}

\subsection{Theorem Statement}

\begin{proposition}[Transcriptome-Guided Learning Bound]
Given readout variables $I$ and $R$ conditioned on cell type $C$, consider a predictor $T: (I, R, C) \mapsto \Delta(\mathcal{D})$ that outputs a distribution over drugs $D$, where $\Delta(\mathcal{D})$ denotes the probability simplex. Define the marginalized predictor that depends only on $(I, C)$:
\begin{equation}
\small 
S_T(D|I, C) = \mathbb{E}_{R|I,C}[T(D|I, R, C)]
\end{equation}

Let $\mathcal{L}[q] = \mathbb{E}_{(I,C,D)\sim\mathcal{P}}[-\log q(D|I, C)]$ denote the log-loss risk for any predictor $q$ that depends only on $(I, C)$. Then $\mathcal{L}[S_T]$ is upper bounded by:
\begin{equation}
\small 
\mathcal{L}[S_T] \leq H(D|I, C) + \mathbb{E}_{(I,R,C)}[\text{KL}(P(D|I, R, C)||T(\cdot|I, R, C))],
\label{eq2}
\end{equation}
where $H(D|I, C)$ is the conditional entropy and the second term is the training objective that can be directly optimized.
\end{proposition}

\subsection{Proof}

\begin{proof}[Proof of Proposition 3.1]
We start by expressing the log-loss risk of an arbitrary predictor $q(D \mid I,C)$ that depends only on $(I,C)$.
By definition,
\begin{align}
\mathcal{L}[q]
&= \mathbb{E}_{(I,C,D) \sim \mathcal{P}}[-\log q(D \mid I,C)] \\
&= \mathbb{E}_{(I,C)} \mathbb{E}_{D \mid I,C}[-\log q(D \mid I,C)] \\
&= \mathbb{E}_{(I,C)} \big[ H(P(D \mid I,C), q(\cdot \mid I,C)) \big],
\end{align}
where $H(P(D \mid I,C), q(\cdot \mid I,C))$ denotes the cross-entropy between the true conditional $P(D \mid I,C)$ and the predictor $q(\cdot \mid I,C)$. Applying cross-entropy decomposition pointwise in $(I,C)$, we obtain:
\begin{align}
\mathcal{L}[q]
&= \mathbb{E}_{(I,C)} \Big[ H(D \mid I,C) 
     + \mathrm{KL}\big(P(D \mid I,C) \,\|\, q(\cdot \mid I,C)\big) \Big] \\
&= H(D \mid I,C) 
   + \mathbb{E}_{(I,C)} \Big[\mathrm{KL}\big(P(D \mid I,C) \,\|\, q(\cdot \mid I,C)\big)\Big].
\label{eq:risk-decomp}
\end{align}
We now specialize to the marginalized predictor
\[
S_T(D \mid I,C)
= \mathbb{E}_{R \mid I,C} \big[ T(D \mid I,R,C) \big].
\]
To bound $\mathbb{E}_{(I,C)} [\mathrm{KL}(P(D | I,C) \,\|\, S_T(\cdot \mid I,C))]$, we express both $P(D \mid I,C)$ and $S_T(D \mid I,C)$ as mixtures over $R$. By the law of total expectation:
$$
P(D\mid I,C)
=\mathbb{E}_{R\mid I,C}[P(D\mid I R,C)],
$$
$$
S_T(D\mid I,C)
=\mathbb{E}_{R\mid I,C}[T(D\mid I,R,C)].
$$
Applying Jensen's inequality for KL divergence, we have:
\begin{align}
&\mathrm{KL}(P(D\mid I,C) \,\|\, S_T(\cdot\mid I,C))
\;\le \\
&\mathbb{E}_{R\mid I,C}\big[\mathrm{KL}(P(D\mid I,R,C)\,\|\,T(\cdot\mid I,R,C))\big].
\end{align}
Taking expectations over $(I,C)$ on both sides, we obtain
\begin{align}
&\mathbb{E}_{(I,C)}\Big[ \mathrm{KL}\big(P(D \mid I,C) \,\|\, S_T(\cdot \mid I,C)\big)\Big]
\\\leq &\mathbb{E}_{(I,C)} \mathbb{E}_{R \mid I,C} 
     \Big[ \mathrm{KL}\big(P(D \mid I, R,C) \,\|\, T(\cdot \mid I,R,C)\big) \Big] \\
= &\mathbb{E}_{(I,R,C)} 
     \Big[ \mathrm{KL}\big(P(D \mid I, R,C) \,\|\, T(\cdot \mid I,R,C)\big) \Big].
\label{eq:kl-bound}
\end{align}

Finally, substituting $q = S_T$ into the risk decomposition~\eqref{eq:risk-decomp} and using the bound~\eqref{eq:kl-bound}, we obtain
{
\small
\begin{align}
&\mathcal{L}[S_T]
= H(D \mid I,C) 
   + \mathbb{E}_{(I,C)}\big[ \mathrm{KL}\big(P(D \mid I,C) \,\|\, S_T(\cdot \mid I,C)\big)\big] \\
&\leq H(D \mid I,C) 
   + \mathbb{E}_{(I,R,C)} 
     \Big[ \mathrm{KL}\big(P(D \mid I, R,C) \,\|\, T(\cdot \mid I,R,C)\big) \Big].
\end{align}
}
Using again $P(D \mid I,R,C) = P(D \mid R,C)$, we can rewrite the bound in the equivalent form:
{
\small
\[
\mathcal{L}[S_T]
\leq H(D \mid I,C) 
+ \mathbb{E}_{(I,R,C)} 
  \Big[ \mathrm{KL}\big(P(D \mid I,R,C) \,\|\, T(\cdot \mid I,R,C)\big) \Big],
\]
}
which is exactly the claimed inequality.
\end{proof}

\subsection{Discussion}

The bound in Proposition~3.1 separates the risk into two conceptually distinct components. The term $H(D \mid I,C)$ captures the intrinsic uncertainty under the restricted information set $(I,C)$ and represents the optimal performance attainable by any predictor that lacks access to $R$ at test time. In contrast, the second term,
\[
\mathbb{E}_{(I,R,C)}\big[\mathrm{KL}\big(P(D \mid I,R,C)\,\|\,T(\cdot \mid I,R,C)\big)\big],
\]
is fully optimizable and quantifies the divergence between the true conditional law and the predictor operating in the augmented information space $(I,R,C)$. By the convexity of KL divergence, marginalizing the predictor over $R$ incurs no worse penalty than the expected KL at the full-information level, establishing that the advantage of training with transcriptomic conditioning is precisely characterized by the reducible part of this discrepancy.

\section{Implementation Details}
\subsection{Single-Cell Foundation Model Fine-tuning}
We fine-tune scGPT~\cite{hao2024large}, a transformer-based single-cell foundation model pretrained on large-scale single-cell atlases, to adapt it for perturbation response prediction on bulk-cell L1000 data. This fine-tuning enables the model to encode cell-type context (basal gene expression), intervention information (drug and dose), and perturbed transcriptomic states in a unified framework.

\subsubsection{Data Featurization}
For each cell line $C$, we use DMSO control samples to represent the unperturbed cellular state. Rather than aggregating controls, we randomly pair each perturbed sample with a DMSO control sample from the same cell line during training. This pairing strategy preserves cellular variability and prevents overfitting to averaged profiles, resulting in basal expression vectors $X_c \in \mathbb{R}^{978}$ sampled from the control distribution. For each drug perturbation, we obtain the perturbed gene expression profile $X_R \in \mathbb{R}^{978}$ from L1000 measurements. These profiles capture the post-perturbation transcriptomic state and serve as the prediction target during fine-tuning.

We obtain molecular representations for each drug $D$ from its chemical structure. Inspired by \citet{yu2025perturbnet}, we use ChemBERTa with pretrained weights as the molecule encoder to encode SMILES. The selected drug representation $\mathbf{f}_D \in \mathbb{R}^{d_f}$ is encoded through a two-layer MLP to obtain a drug embedding $\mathbf{e}_D \in \mathbb{R}^{d_e}$. We encode the perturbation dose $\delta_R$ using a learned embedding layer. Doses are encoded as continuous values normalized to $[0, 1]$. A single-layer MLP maps the dose to an embedding $\mathbf{e}_{\delta} \in \mathbb{R}^{d_\delta}$. The drug and dose embeddings are combined through concatenation to form a unified intervention embedding $\mathbf{e}\in \mathbb{R}^{d_e+d_\delta}$.

\subsubsection{scFM Encoding}
The basal gene expression $X_c$ is fed into the pretrained scGPT encoder to obtain a cell-type representation: $$Z_c = f_{\text{scFM}}(X_c) \in \mathbb{R}^{e}$$ where $e$ is the scFM embedding dimension. The basal cell embedding $Z_c$ and intervention embedding $\mathbf{e}$ are concatenated to form the complete input representation for the prediction task: $Z_{\text{combined}} = [Z_c | \mathbf{e}] \in \mathbb{R}^{e + d_e+d_\delta}$
This design allows the model to condition perturbation predictions on both the cellular context and the specific intervention (drug identity and dosage). Then we use a decoder network to predict the perturbed gene expression $\hat{X_R}$ directly from the combined encoding. The decoder is a MLP.
We minimize the mean squared error (MSE) between predicted and observed perturbed gene expression: $$\mathcal{L}_{\text{scFM}} = \frac{1}{N} \sum_{i=1}^{N} ||\hat{X_R}_i - {X_R}_i||_2^2,$$ where $N$ is the number of training samples. This direct prediction objective allows the model to learn the complete transcriptomic response to perturbations, including both the magnitude and direction of expression changes across all landmark genes.

\subsubsection{Training Details}
We fine-tune all parameters of scGPT encoder along with the decoder, allowing the foundation model to adapt to perturbation-specific patterns while retaining its general gene expression knowledge. We split drug-cell line combinations (not individual samples) into train/val/test sets to ensure evaluation on held-out drugs. 

After fine-tuning, we extract two types of embeddings for use in TIDE's teacher model: Cell-type encoding $\mathbf{h}_C = Z_c$ and Dose encoding $\mathbf{h}_{\delta_r} = \mathbf{e}_{\delta}.$ We encode raw perturbed expression directly in the teacher.

\subsection{Network Architectures}
We provide detailed architectural specifications for all components of TIDE and baseline methods. Unless otherwise specified, all dimensions are matched to the vision backbone's hidden dimension $d$.

\subsubsection{Vision Transformer Backbones}
We evaluate TIDE on two vision transformer architectures: standard Vision Transformer (ViT) and Channel-Agnostic Vision Transformer (CA-ViT). For ViT, we adopt the configurations from~\cite{dosovitskiy2020image}. ViT-Small consists of 12 transformer layers with hidden dimension 384 and 6 attention heads, while ViT-Base uses 12 layers with hidden dimension 768 and 12 attention heads. Both variants use a patch size of 16×16, learnable position embeddings, and GELU activations with LayerNorm. The MLP expansion ratio is set to 4 following standard practice. Input images are resized to 224×224 resolution. We initialize ViT encoders using random initialization.

Channel-Agnostic Vision Transformer (CA-ViT)~\cite{kraus2024masked} is specifically designed for multi-channel microscopy imaging. Unlike standard ViT which treats all channels as a single RGB-like input, CA-ViT processes each imaging channel independently through channel-specific patch embedding layers before fusing information across channels in deeper layers. This design is particularly suitable for Cell Painting and RxRx data, which contain 5-6 fluorescent channels with distinct biological meanings. We use CA-ViT-Small with the same depth and hidden dimensions as ViT-Small. The channel fusion is performed through learned aggregation after the initial patch embedding stage, allowing the model to weight different channels adaptively. Despite slightly higher parameter counts due to per-channel embeddings, CA-ViT achieves comparable or better performance than standard ViT on cellular imaging tasks under sufficient training.

\subsubsection{Chemistry-Aware Codebook Construction}
The chemistry-aware codebook transforms drug molecular representations into a structured prototype space. Given a drug's molecular representation $\mathbf{f}_D \in \mathbb{R}^{d_f}$, we apply a two-layer MLP followed by L2 normalization to project it into the same embedding space as the teacher and student encoders. We experimented with multiple molecular representation methods, including FCFP, RDKit molecular descriptors, pretrained graph isomorphism networks (GIN), and ChemBERTa embeddings. Performance comparisons are reported in Appendix D.1.  This produces the prototype matrix $V \in \mathbb{R}^{K \times d}$, where $K$ is the number of training drugs. All parameters including the prototype vectors are optimized end-to-end during training.

\subsubsection{Teacher-Student Network}
The teacher network integrates three information sources, images, transcriptomics, and cell type, to produce distributions over the codebook. The image encoder uses the vision transformer backbones described in Section B.2.1, extracting features by averaging image tokens to obtain $\mathbf{h}_I \in \mathbb{R}^d$. To incorporate dose information from images, we apply FiLM-style conditioning~\cite{perez2018film} where a small dose predictor network (two linear layers of dimensions $d \to 128 \to 1$ with ReLU) estimates the dose effect, and then modulates the image features through learned scale and shift parameters.

We encode the perturbed gene expression $X_R \in \mathbb{R}^{978}$ through a MLP, and concatenate the embedding with its dose embedding $e_\delta$. Then we use a projector to map the concatenated embedding to obtain $\mathbf{h}_R$. The cell-type encoder uses the scFM to encode basal gene expression $X_c$, producing $\mathbf{h}_C \in \mathbb{R}^d$ after projection. The fusion network concatenates the three feature vectors $[\mathbf{h}_I | \mathbf{h}_R | \mathbf{h}_C] \in \mathbb{R}^{3d}$ and processes them through a two-layer MLP with intermediate dimension $2d$. Each layer is followed by BatchNorm, with ReLU activation and dropout (rate=0.1) after the first layer. The output fused representation $\mathbf{h}_t \in \mathbb{R}^d$ is then used to compute cosine similarities with all codebook prototypes. A temperature-scaled softmax with $\tau = 0.2$ converts these scores into the teacher's probability distribution $P_t$ over the codebook.

The student network operates solely on images. Crucially, the student shares the same image encoder as the teacher, with full parameter sharing during training. This weight sharing ensures that improvements in the teacher's image understanding directly benefit the student. After extracting image features by vision encoder, we apply a projection head consisting of two linear layers. The final output is L2 normalized to produce the student representation $\mathbf{h}_s \in \mathbb{R}^d$. The student predicts codebook distributions using the same cosine similarity and temperature-scaled softmax as the teacher, but with a lower temperature $\tau_s = 0.07$ to produce sharper predictions. In practice, to accelerate training convergence, we supplement the distillation objective with direct supervision: the student's predicted distribution is also trained with cross-entropy loss against ground-truth drug labels. This auxiliary supervision provides a stronger learning signal in early training stages before the teacher has learned meaningful soft targets, and helps stabilize optimization throughout training.

\subsubsection{Self-Supervised Learning Components}
For DINO~\cite{caron2021emerging}, we use the standard multi-crop strategy with 2 global crops at 224×224 resolution and 4 local crops at 96×96 resolution. The teacher network uses exponential moving average (EMA) updates with momentum 0.996. Temperature parameters are set to 0.07 for the teacher and 0.1 for the student, with centering momentum 0.9. For MAE~\cite{he2022masked}, we use the same encoder backbones with a lightweight 8-layer decoder (hidden dimension 512) and mask ratio 0.75. Both SSL methods use the standard patch size of 16×16.

\subsubsection{Baseline Method Architectures}
For CL(D), which aligns images to drug molecular representations, we use the same vision backbones as TIDE for the image encoder. To maintain consistency with the drug encoder used in our scFM fine-tuning (Appendix B.1), we employ ChemBERTa as the drug encoder. ChemBERTa provides pretrained contextualized representations of molecular SMILES strings, which we extract from the [CLS] token. Both image and drug features are projected to 256-dimensional embeddings through separate projection heads before computing InfoNCE loss with temperature 0.07. CL(R) aligns images to RNA-seq profiles using contrastive learning. The image encoder and projection head are identical to CL(D). For RNA encoding, contrastive learning typically requires large batch sizes to provide sufficient negative samples, which poses GPU memory constraints when using large foundation models like scGPT. To enable efficient training with batch sizes of 512, we instead use a lightweight two-layer MLP that maps the 978-dimensional L1000 profiles to dimension $d$. The encoded RNA features are then projected to 128 dimensions using the same projection head as images before computing the InfoNCE loss with temperature 0.07. For hybrid variants combining contrastive learning with DINO (CL(D)+DINO, CL(R)+DINO), we add the DINO loss with weight 0.2 to the contrastive loss.

\begin{figure*}[tbp]
    \centering
    \begin{subfigure}[b]{0.245\textwidth}
        \centering
        \includegraphics[width=\linewidth]{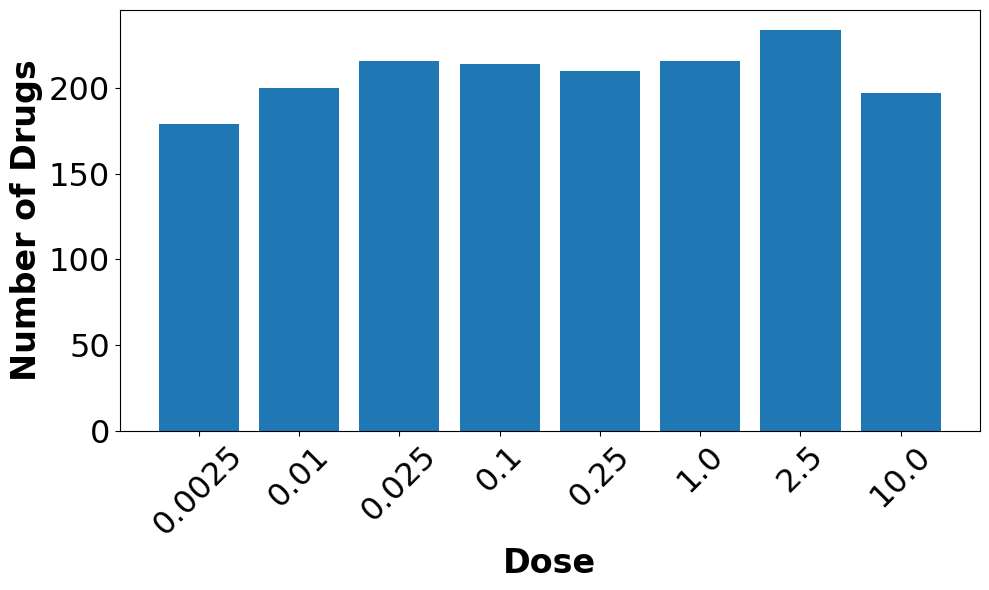}
        \caption{RxRx3}
        \label{fig:sub1}
    \end{subfigure}
    \begin{subfigure}[b]{0.245\textwidth}
        \centering
        \includegraphics[width=\linewidth]{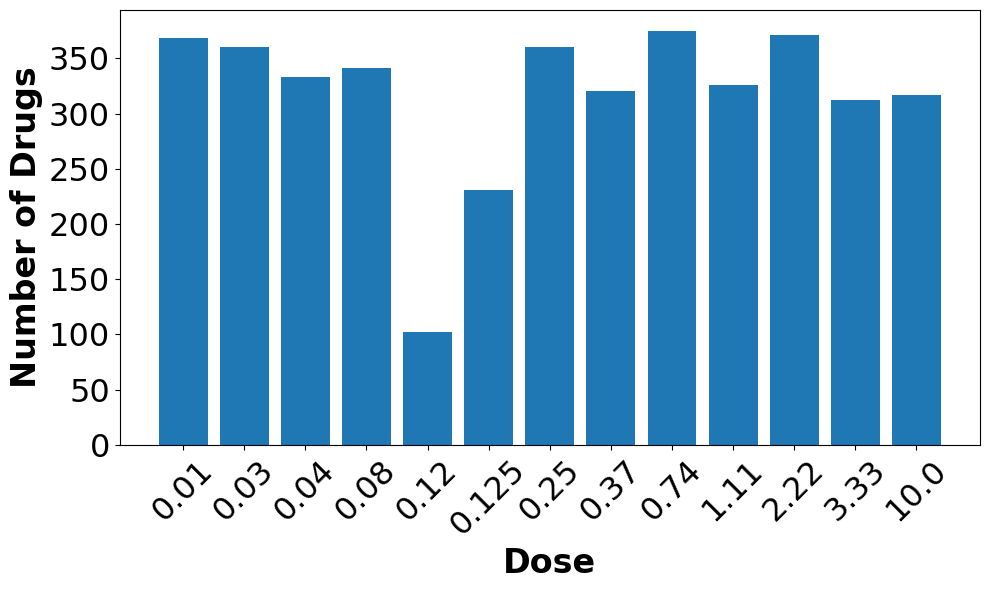}
        \caption{HUVEC}
        \label{fig:sub2}
    \end{subfigure}
    \begin{subfigure}[b]{0.245\textwidth}
        \centering
        \includegraphics[width=\linewidth]{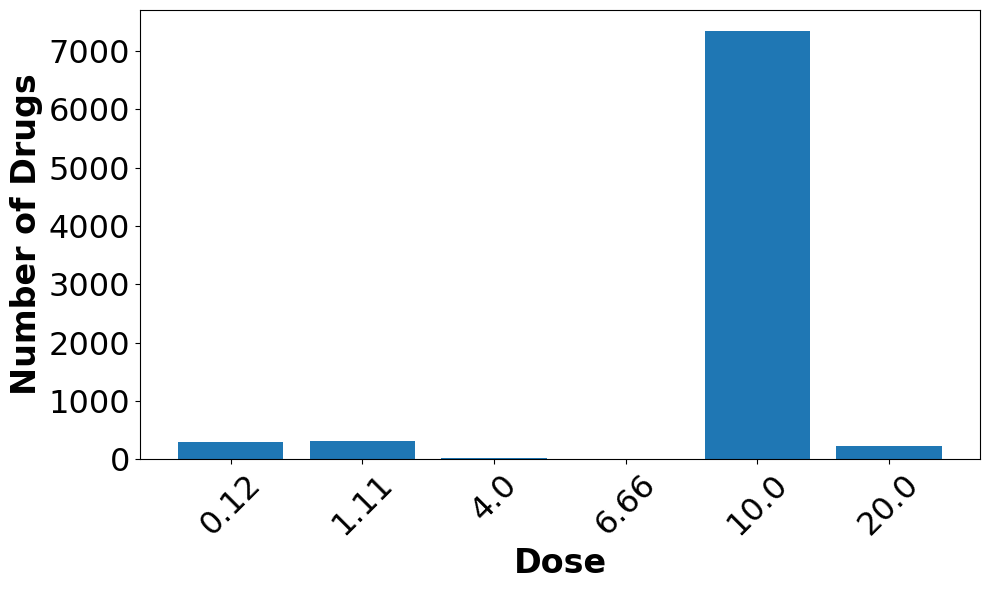}
        \caption{U2OS}
        \label{fig:sub3}
    \end{subfigure}
    \begin{subfigure}[b]{0.245\textwidth}
        \centering
        \includegraphics[width=\linewidth]{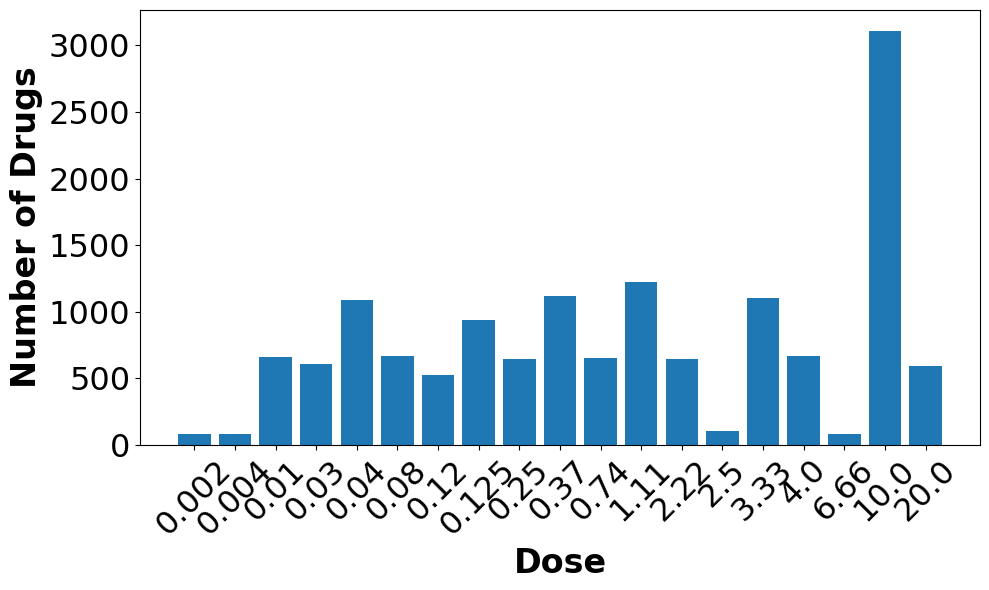}
        \caption{A549}
        \label{fig:sub4}
    \end{subfigure}
    \caption{Dose ($\mu$M) distribution across imaging and transcriptomic datasets. (a) Dose distribution in RxRx3 imaging data. (b-d) Dose distributions in L1000 transcriptomic profiles for HUVEC, U2OS, and A549 cell lines.}
    \label{fig:four-in-a-row}
\end{figure*}
\section{Dataset Details}
\subsection{Data Preprocessing}

CPG-Pilot and CPG-12 use the standardized Cell Painting protocol with five fluorescent channels targeting different cellular compartment. They provide 16-bit multi-channel TIFF images\footnote{https://github.com/broadinstitute/cellpainting-gallery}. RxRx3 uses a six-channel imaging protocol on HUVEC cells\footnote{https://www.rxrx.ai/rxrx3-core}. For Cell Painting datasets (CPG-Pilot and CPG-12), we apply flatfield illumination correction to compensate for uneven lighting and optical aberrations across the field of view. This preprocessing step ensures uniform intensity distributions and removes spatial artifacts that could confound biological signals. To handle intensity variations and remove outliers, we clip extreme pixel values at the 0.01\% and 99.9\% percentiles per channel before rescaling intensities to the [0, 1] range. Images are resized to 224×224 resolution to match the Vision Transformer input requirements.

L1000 data is obtained from the LINCS database in GCTX (Gene Expression Connectivity) format, containing expression measurements for 978 landmark genes across multiple cell lines and perturbations. Each profile represents a bulk-cell population average from treated wells. We apply minimal filtering to L1000 profiles for quality control. The primary criterion is dose-level sample size: we remove dose conditions with fewer than 5 replicate samples to ensure stable statistics.

\subsection{Details of Drug-Target Discovery}
For the unsupervised drug-target gene discovery task, we follow the standard preprocessing protocol established by~\citet{kraus2025rxrx3}. This protocol consists of two key steps: crop aggregation and batch alignment. For RxRx3, each 512×512×6 microscopy image produces four tiled 256×256×6 crops. We compute embeddings for each crop independently using the trained image encoder, then mean-aggregate embeddings from the same well to obtain a single well-level representation. This aggregation reduces within-well variance and provides more stable perturbation signatures. To align embeddings across experimental batches and remove systematic technical variation, we employ Principal Component Analysis with Centering and Scaling (PCA-CS). First, we fit a PCA transformation on all DMSO control sample embeddings without dimensionality reduction, capturing the control distribution's variance structure. This PCA transformation is then applied to all embeddings (both controls and perturbed samples). Next, for each experimental batch, embeddings are centered and scaled relative to the batch's DMSO control statistics. This procedure standardizes embeddings and minimizes batch effects while preserving biological perturbation signals. Following alignment, we compute drug-target association metrics on the batch-corrected embeddings.

\subsection{Dose Distribution}
\autoref{fig:four-in-a-row} presents the dose distributions across imaging and L1000 transcriptomic datasets. The distributions reveal substantial heterogeneity both within and across modalities, underscoring the importance of explicit dose modeling in our framework.
For RxRx3 imaging data (\autoref{fig:sub1}), doses are relatively uniformly distributed across the range from 0.0025 $\mu$M to 10.0 $\mu$M, with each dose bin containing 150-250 drug perturbations. This balanced coverage enables evaluation across different perturbation strengths. In contrast, the L1000 transcriptomic profiles show markedly different dose distributions across cell lines. HUVEC exhibits diverse dose coverage spanning 0.01 $\mu$M to 30.0 $\mu$M, with notable peaks around 0.12 $\mu$M, 0.37 $\mu$M, and multiple higher doses. U2OS is heavily concentrated at 10.0 $\mu$M, with over 6000 drugs at this single dose and sparse coverage elsewhere, reflecting a standardized screening dosage for this cell line. A549  shows the most dispersed distribution, with substantial measurements across 0.003-30.0 $\mu$M and a pronounced peak at 30.0 $\mu$M containing over 3000 drugs.
These heterogeneous dose distributions have critical implications for multimodal learning. First, when pairing imaging and transcriptomic data, exact dose matches are rare. For instance, RxRx3's peak doses (e.g., 1.0 $\mu$M) may not align with HUVEC's common doses (e.g., 0.12 $\mu$M or 10.0 $\mu$M), creating systematic dose mismatches. Second, the dose-dependent nature of drug effects means that treating different doses of the same compound as equivalent, as in identity-based alignment methods, discards critical information about perturbation magnitude. While our framework addresses this challenge through explicit dose conditioning in the teacher module. The teacher encodes dose information separately for images ($\delta_I$) and transcriptomics ($\delta_R$), allowing it to reason about dose-specific effects even when doses differ across modalities. The scFM fine-tuning further enables disentanglement of dose-dependent transcriptomic responses from cell-type identity. This design is particularly important for HUVEC and A549, where the diverse dose distributions mean that most paired samples have dose mismatches. 

\section{Additional Experimental Results}
In this section, we provide additional experimental results that complement the main paper. We present ablation studies on codebook design choices, detailed results and analysis of dose-dependent effects, and validation of fine-tuned scFM.

\begin{table}[t]
    \centering
    \begin{tabular}{l|cc|cc}
    \toprule
     Dataset    & \multicolumn{2}{c|}{RxRx3} & \multicolumn{2}{c}{CPG-12}\\
     \midrule
     Metric & Top-1 & Top-5 &  Top-1 & Top-5 \\ 
     \midrule
     None & 5.20 & 15.82 & 39.39 & 65.06 \\
     GIN & 5.72 & 16.28 &41.26 & 69.03\\
     RDKit & 5.57 & 16.35 & 40.09 & 65.10\\
     ChemBERTa & 5.66 & 16.80 & 41.00 & 68.88\\
     FCFP  & {5.62} & {16.86} & {42.09} & {71.07}\\
    Fixed FCFP &5.11 & 15.71 & 40.13 & 66.56\\
     \bottomrule
    \end{tabular}
    \caption{Performance with different drug molecular representations on CPG-12.}

    \label{tab:codebook}
\end{table}

\subsection{Ablation Studies on Drug Codebook}
We investigate key design choices for the chemistry-aware codebook to understand which components contribute to TIDE's performance. \autoref{tab:codebook} presents results on RxRx3 and CPG-12 with different codebook configurations using ViT as the vision encoder.

The ``Non'' baseline directly predicts drug identity through a classification loss without using a codebook structure. Compared to this baseline, all codebook-based variants show improved performance, validating that organizing drugs in a prototype space facilitates soft similarity-based learning.
We test four molecular representation methods, GIN, RDKit descriptors, ChemBERTa, and FCFP, each paired with a learnable MLP projection and end-to-end optimization. Performance across these representations is relatively comparable, with all achieving meaningful improvements over the ``Non'' baseline. This suggests that TIDE's framework is robust to the choice of initial molecular encoding, as the learnable codebook adapts representations during training to encode mechanistic similarity beyond pure chemical structure. The comparison between learned FCFP and fixed FCFP codebook is also informative. The fixed variant, where prototypes are frozen at their ECFP initialization, performs substantially worse than the learned version, and even falls below some other learned representations. This validates our design that the codebook must be optimized end-to-end to move beyond chemical similarity and encode the mechanistic relationships learned from transcriptomic guidance. The fixed codebook cannot capture pathway-level similarities, confirming that learning is essential rather than the framework simply exploiting chemical shortcuts. Overall, FCFP with learnable projection achieves the best performance across both datasets. Based on these results, we adopt FCFP as the molecular representation for all experiments in main paper.

\begin{figure}[t]
    \centering
    \includegraphics[width=0.95\linewidth]{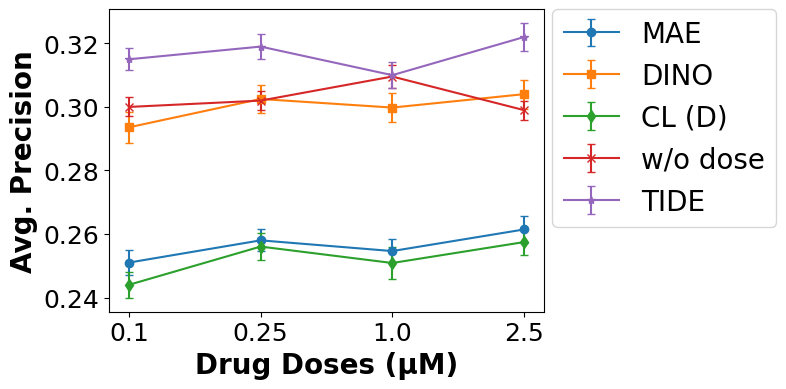}
    \caption{Drug-target discovery performance across dose levels on RxRx3. Average Precision (AP) for drug-target gene retrieval at four curated dose levels using ViT backbone.}
    \label{fig:doses}
\end{figure}
\subsection{Dose-Specific Performance and Ablation}
To assess how perturbation dose affects drug-target discovery performance and validate the necessity of explicit dose conditioning, we analyze results on RxRx3 stratified by dose level. Based on the original RxRx3-core annotations, we curate four dose levels with significant drug-target associations: 0.1, 0.25, 1.0, and 2.5 $\mu$M. \autoref{fig:doses} presents Average Precision (AP) across these doses for multiple methods using ViT backbone. Performance trends are consistent across doses for all methods, with 2.5 $\mu$M yielding the relatively highest AP values. The ``w/o dose'' ablation removes dose conditioning from both the scFM encoder and the teacher's image dose modeling, treating all doses of the same drug equivalently. We can find that its performance drops substantially across all dose levels. This validates that explicit dose disentanglement is critical for leveraging the heterogeneous dose distributions in weakly paired data, enabling the model to reason about dose-specific effects rather than collapsing them through identity-based matching.

\subsection{scFM Fine-tuning Results}
To validate the quality of our fine-tuned scFM encoder, 
we report its performance on perturbation response prediction. The model is trained to predict perturbed gene expression given basal expression, drug representation, and dose. As shown in \autoref{tab:scfm}, the fine-tuned scFM achieves $R^2$ between 0.70-0.84 across cell lines, demonstrating strong predictive performance on novel compounds.  These results validate that the scFM successfully learns to encode cell-type identity from basal gene expression, capture dose-dependent perturbation magnitude, and  predict pathway-level transcriptomic changes. This learned capability is essential for TIDE's teacher model, which relies on the scFM to provide mechanism-aware encodings of cell context, dose information, and perturbed RNA states during distillation.

\begin{table}[t]
     \resizebox{0.48\textwidth}{!}{
         \centering
    \begin{tabular}{l|cc|cc|cc}
    \toprule
     Dataset & \multicolumn{2}{c}{RxRx3} & \multicolumn{2}{c}{CPG-Pilot } & \multicolumn{2}{c}{CPG-12}  \\
     Metric $\uparrow$  & $R^2$ & LogFC &$R^2$ & LogFC & $R^2$ & LogFC \\
     \midrule
     Fine-tuned &0.835 & 0.622& 0.815 & 0.572 & 0.699 &0.850 \\
    \bottomrule
    \end{tabular}}
    \caption{scFM fine-tuning performance on L1000 perturbation response prediction. Each column corresponds to a scFM model trained on L1000 data matching the cell lines used in the imaging dataset: RxRx3 (HUVEC), CPG-Pilot (U2OS+A549), and CPG-12 (U2OS). }
    \vspace{-12pt}
    \label{tab:scfm}
\end{table}

\section{Background on L1000 and Cell Morphology}
L1000 is a high-throughput gene expression profiling technology developed by the Broad Institute as part of the Library of Integrated Network-Based Cellular Signatures (LINCS) program~\cite{subramanian2017next}. Different from traditional RNA-seq which measures the entire transcriptome, L1000 uses a cost-effective approach that directly measures expression of 978 carefully selected ``landmark'' genes using a bead-based detection system. These landmark genes are chosen to be highly informative: their expression patterns can be used to computationally infer the expression of the remaining around 80\% of protein-coding genes through statistical models trained on large-scale gene expression databases. Key characteristics of L1000 data include: (i) bulk-cell measurements that represent population-averaged gene expression from thousands of cells per well, smoothing out single-cell heterogeneity; (ii) pathway-level resolution capturing coordinated changes in gene programs rather than individual gene fluctuations; (iii) dose-dependent responses where the same perturbation at different concentrations produces graded transcriptomic changes; and (iv) cell-type specificity where identical perturbations can elicit distinct transcriptomic profiles across different cell lines due to varying basal pathway activity and regulatory networks.

Cell Painting~\cite{bray2016cell} is a standardized multiplexed fluorescence microscopy protocol designed to capture rich morphological information about cellular state. The protocol uses six fluorescent dyes that target distinct cellular components: Hoechst 33342 (nucleus/DNA), concanavalin A/phalloidin (endoplasmic reticulum and cytoskeleton), SYTO 14 (nucleoli and cytoplasmic RNA), wheat germ agglutinin (Golgi and plasma membrane), MitoTracker Deep Red (mitochondria), and phalloidin (F-actin). By imaging these five channels (some dyes are combined) at high resolution, Cell Painting generates multi-channel images that reveal detailed morphological features including cell shape, size, texture, organelle distribution, and sub-cellular compartmentalization~\cite{chandrasekaran2024three}. Cell Painting data is characterized by: (i) scalability to millions of cells and thousands of perturbations per experiment at relatively low cost; (ii) phenotypic richness capturing morphological changes invisible to the naked eye but indicative of cellular state; (iii) perturbation sensitivity where subtle compound effects manifest as quantifiable morphological shifts.

While both technologies profile cellular responses to perturbations, they capture fundamentally different aspects of cellular state. L1000 provides direct molecular evidence of pathway activation and gene regulatory changes, making it well-suited for identifying mechanisms of action (MoA) at the molecular level. Cell Painting captures integrated phenotypic outcomes, which can reveal subtle perturbation effects and distinguish compounds with similar molecular targets but different cellular contexts. This complementarity has been quantitatively demonstrated in MoA prediction studies. For instance, \citet{way2022morphology} compared single-modality versus multi-modal profiling, and found that combining Cell Painting and L1000 enables detection of approximately 69\% of mechanism categories, whereas using Cell Painting images or L1000 alone detects only 19\% or 24\% of mechanisms, respectively. This synergy arises because some MoAs produce strong transcriptomic signatures but subtle morphological changes, while others have pronounced phenotypic effects with modest gene expression alterations. The demonstrated complementarity has motivated extensive research in multi-modal learning approaches published at recent machine learning and computer vision conferences. Contrastive alignment methods \cite{sanchez2023cloome,zheng2024cross,liu2024learning} are widely adopted: models independently encode images and gene expression, then align representations through contrastive losses that pull together embeddings from the same perturbation while pushing apart different perturbations. Beyond the standard contrastive learning, InfoCORE~\cite{wang2023removing} addresses batch effect confounders in contrastive multi-modal learning. While MINER~\cite{rao2025multi} extends this direction by incorporating hard negative mining strategies to improve cross-modal alignment robustness. Another line of works focuses on cross-modal knowledge transfer, which aims to enhance single-modality representations through distillation. ~\citet{bendidi2025cross} proposed using Cell Painting images as a ``teache'' to guide transcriptomic ``student'' encoders via knowledge distillation, leveraging abundant imaging data to improve gene expression representations. Conversely, other work \cite{navidi2024morphodiff,zhang2025cellflux} explores generative modeling: given perturbed transcriptomic profiles, diffusion models synthesize corresponding cellular morphology images, enabling data augmentation and mechanistic visualization.

These prior works generally focus on either learning joint representations for downstream molecular tasks, with images serving as auxiliary signals, or improving transcriptomic or drug representations through image-based regularization. Our work differs by targeting image representation learning as the primary objective, using transcriptomics as privileged information during training to guide image encoders toward mechanistic understanding, enabling purely image-based drug discovery at deployment while benefiting from transcriptomic knowledge.


\end{document}